\documentclass{article}

\usepackage[authoryear]{natbib}
\usepackage[final]{neurips_2020}
\usepackage{algorithm} 
\usepackage{algorithmic}  
\usepackage[utf8]{inputenc} 
\usepackage[T1]{fontenc}    
\usepackage{hyperref}       
\usepackage{url}            
\usepackage{booktabs}       
\usepackage{amsfonts}       
\usepackage{nicefrac}       
\usepackage{microtype}      
\usepackage{url}  
\usepackage{graphicx}
\usepackage{amsfonts}
\usepackage{amsmath,amssymb}
\usepackage{wrapfig}
\usepackage{lipsum}
\PassOptionsToPackage{authoryear}{natbib}
\setcitestyle{authoryear}
\usepackage{booktabs,array}

\newcount\rowc

\makeatletter
\def\ttabular{%
\hbox\bgroup
\let\\\cr
\def\rulea{\ifnum\rowc=\@ne \hrule height 1.0pt \fi}
\def\ruleb{
\ifnum\rowc=1\hrule height 1.3pt \else
\ifnum\rowc=8\hrule height \heavyrulewidth 
   \else \hrule height \lightrulewidth\fi\fi}
\valign\bgroup
\global\rowc\@ne
\rulea
\hbox to 8em{\strut \hfill##\hfill}%
\ruleb
&&%
\global\advance\rowc\@ne
\hbox to 5em{\strut\hfill##\hfill}%
\ruleb
\cr}
\def\endttabular{%
\crcr\egroup\egroup}

\usepackage{hyperref}       
\usepackage{url}            
\usepackage{booktabs}       
\usepackage{amsfonts}       
\usepackage{nicefrac}       
\usepackage{microtype}      
\usepackage{url}  
\usepackage{graphicx}
\usepackage{amsfonts}
\usepackage{amsmath,amssymb}
\usepackage{amsthm}

\usepackage{graphicx}
\usepackage{subcaption}

\DeclareMathOperator*{\E}{\mathbb{E}}
\DeclareMathOperator*{\argmax}{arg\,max}

\newtheorem{proposition}{Proposition}

\title{Behavior-Guided Actor-Critic: Improving Exploration via Learning Policy Behavior Representation for Deep Reinforcement Learning}

\author{Ammar Fayad\\
  Massachusetts Institute of Technology\\
  \texttt{afayad@mit.edu} \\
   \And 
   Majd Ibrahim \\
   Higher Institute for Applied Sciences and Technology\\
   \texttt{majd.ibrahim@hiast.edu.sy} \\
}

\begin{document}

\maketitle

\begin{abstract}
  In this work, we propose Behavior-Guided Actor-Critic (BAC), an off-policy actor-critic deep RL algorithm. BAC mathematically formulates the behavior of the policy through autoencoders by providing an accurate estimation of how frequently each state-action pair was visited while taking into consideration state dynamics that play a crucial role in determining the trajectories produced by the policy. The agent is encouraged to change its behavior consistently towards less-visited state-action pairs while attaining good performance by maximizing the expected discounted sum of rewards, resulting in an eﬃcient exploration of the environment and good exploitation of all high reward regions. One prominent aspect of our approach is that it is applicable to both stochastic and deterministic actors in contrast to maximum entropy deep reinforcement learning algorithms. Results show considerably better performances of BAC when compared to several cutting-edge learning algorithms.
\end{abstract}
\section{Introduction}
Reinforcement learning (RL) studies an agent interacting with an environment that is unknown at first. One central challenge in RL is balancing between exploration—moving towards new regions of the problem space and exploitation—taking the best sequence of actions based on previous experiences. Without enough exploration, the agent might fall quickly into a local optimum especially in deceptive environments. Exploration in large RL tasks is still most often performed using simple exploration strategies such as epsilon-greedy \citep{Mnih2015HumanlevelCT}and i.i.d./correlated Gaussian noise \citep{Lillicrap2015ContinuousCW, Wu2017ScalableTM}, which can be inconvenient in more complex settings. Some of the other mathematically-proven exploration methods like PAC-MDP algorithms (MBIE-EB \citep{Strehl2008AnAO}), and Bayesian algorithms (BEB \citep{Kolter2009NearBayesianEI}) are based on counting state-action visitations and assigning a bonus reward to less visited pairs.

Different from other approaches, the BAC agent is assigned an encoder-decoder function that learns feature representations. Specifically, this function extracts the features of the data generated by the policy when performing rollouts. These features are then used to define the policy behavior as they contain information about the most frequently visited regions of both state and action spaces. Our method follows the off-policy actor-critic approach  \citep{degris2012off}, and further aims to consistently change the behavior of the agent resulting in a policy that not only learns to perform effectively in a certain region of the problem spaces but also tries to figure out different solution configurations that may lead to better performances.

This paper starts by introducing the formal concepts used later in the algorithm, we further provide a theoretical analysis and prove that our method converges to the optimal policy regardless of the initial parameterization of the policy network. In our experiments, we evaluate our algorithm on five continuous control domains from OpenAI gym \citep{brockman2016openai}, we also evaluate BAC using a deterministic actor and compare it to the standard stochastic BAC in terms of performance and stability.

\section{Related Work}
A plenty of exploration methods have been widely investigated; a widespread idea is to approximate state or state-action visitation frequency \citep{Tang2017ExplorationAS,Bellemare2016UnifyingCE, Ostrovski2017CountBasedEW, Lang2012ExplorationIR}.  A number of prior methods also examine learning a dynamics model to predict future states \citep{Pathak2017CuriosityDrivenEB,Sorg2010VarianceBasedRF, Lopes2012ExplorationIM, geist2010managing, araya2012near}. Nevertheless, they might be easily deceived by an unpredictable transition of visual stimulus; such phenomenon is called the $Noisy-TV$ problem \citep{burda2018large}. Another interesting (line) of exploration methods is based on the idea of optimism in the face of uncertainty \citep{brafman2002r, Jaksch2008NearoptimalRB, Kearns2004NearOptimalRL}. These methods achieve efficient exploration and offer theoretical guarantees. However, they quickly become intractable as the size of state spaces increases. Some works utilize domain-speciﬁc factors to improve exploration \citep{doshi2010nonparametric, Schmidhuber2006DevelopmentalRO, Lang2012ExplorationIR}. That being said, defining domain-dependent functions (e.g. behavior characterization \citep{Conti2018ImprovingEI}) requires prior knowledge of the environment and might prove to be inaccurate as well as impractical when dealing with tasks that require learning multiple skills. However, we draw inspiration from these methods to formalize a function that provides accurate information about the visitation frequency and simultaneously enable us to apply it to various high dimensional and continuous deep RL benchmarks. Moreover, our method combines actor-critic architecture with an off-policy formulation \citep{Haarnoja2018SoftAO, Lillicrap2015ContinuousCW, chen2019off} that permits the reuse of any past experience for better sample efficiency. The actor and critic, on the other hand, are optimized jointly to conform with large-scale problems. We also employed two critic networks to approximate the value of the policy as proposed by \citep{hasselt2010doubleQ,fujimoto2018addressing} to mitigate positive bias in the policy improvement step that might degrade the performance. 
\section{Preliminaries}
\label{sec-back}


\subsection{The Reinforcement Learning Problem}
\label{sub-rl}

We address policy learning in continuous action spaces and continuous state spaces, considering an inﬁnite-horizon Markov decision process (MDP), deﬁned by the tuple $(\mathcal{S},\mathcal{A},p,r)$ with an agent interacting with an environment in discrete timesteps. At each timestep $t$, the agent in state $s_t\in\mathcal{S}\subseteq \mathbb{R}^n$ performs an action $a_t\in\mathcal{A}\subseteq\mathbb{R}^m$ and receives a scalar reward $r_t$. Every action causes a state transition in the environment from state $s_t$ to a new state $s_{t+1}$, governed by the state transition probability $p:\mathcal{S}\times\mathcal{S}\times\mathcal{A}\to [0, \infty)$ which is typically unknown to the agent.
Actions are chosen by a policy, $\pi$, which maps states to a probability distribution over the actions $\pi:\mathcal{S}\to\rho(\mathcal{A})$, under which the agent tries to maximize its long-term cumulative rewards, as described below.
\begin{equation}
J(\pi)= \E_{(s_t\sim p,a_t\sim \pi)}[ \sum_{t=0}^{\infty} \gamma^t r(s_t,a_t)]
\label{eq-lt-cum-rew}
\end{equation}
\noindent
with a discounting factor $\gamma\in[0,1)$. The goal in reinforcement learning is to learn a policy, $\pi^*= \max _{\pi} J(\pi)$.
\subsection{Autoencoders}
An autoencoder is a neural network that is trained to attempt to copy its input to its output. Generally, the autoencoder tries to learn a function $\phi(x)\approx x$, by minimizing a loss funcion\\ (e.g. $\mathcal L= ||\phi (x)-x||^2_2$)\citep{Goodfellow-et-al-2016}. In practice, the autoencoders are unable to learn to copy perfectly. However, their loss functions allow them to copy input that resembles the training data. That is, when they are trained on certain data distributions, they fail to copy input sampled from different distributions $\mathcal D'$, i.e. the distance $||\phi(x')-x'||^2_2$ is relatively large, where $x'\sim \mathcal D'$. This fact is central to defining our proposed behavior function. 

\section{Policy Behavior}
The behavior of a reinforcement learning agent is closely related to the trajectory distribution induced by its policy. However, defining the behavior as a function of the trajectory distribution is not practical since it requires prior knowledge of the state dynamics.

To avoid this pitfall, we introduce a novel domain-independent approach to evaluate the behavior of a policy $\pi$ using neural networks. Furthermore, we initialize a comparatively large autoencoder network $\phi_{\varphi}$, with parameters $\varphi$, and train it on data generated from policy rollouts. The main idea is that after enough training, $\phi_{\varphi}$ yields good representations of trajectories frequently induced by the policy, while poorly representing odd trajectories. Additionally, given that trajectory distribution is affected by the transition probabilities, the autoencoder $\phi_{\varphi}$ provides accurate trajectory representations while taking the dynamics of the environment into consideration.

In practice, however, using a whole trajectory as an input to the neural network can be computationally expensive and difficult to scale to environments with high dimensional continuous state and action spaces as the ones in our experiments.

Alternatively, we build our algorithm on learning the representations of state-action pairs since it is compatible with the classical reinforcement learning framework as demonstrated in the next section. Hence, we can define the behavior value of a state-action pair $(s,a)$ as following:
\begin{equation}
\psi^{\pi_{\theta}}(s,a)= ||\phi^{\pi_{\theta}}_{\varphi}(s,a) - [s,a]||^2_2
\label{behavior-value} \end{equation}
Where $\phi^{\pi_{\theta}}_{\varphi}$ is the autoencoder that corresponds to $\pi_{\theta}$ and $[s,a]$ is the concatenated vector of $(s,a)$.

Note that, as stated before, if $(s,a)$ is frequently visited by ${\pi_{\theta}}$ then $\psi^{\pi_{\theta}}(s,a)$ is relatively small. Consequently, this motivates the formal definition of the policy behavior function $\Psi(\pi_{\theta})$:
\begin{equation}
\Psi(\pi_{\theta})= \mathbb{E}[\psi^{\pi_{\theta}}(s,a)]= \mathbb{E}_{s\sim p, a \sim\pi_{\theta}} [||\phi^{\pi_{\theta}}_{\varphi}(s,a) - [s,a]||_2^2]
\label{policy-behavior} \end{equation} \noindent\\
For simplicity of notation, we will write $\pi, \phi, \psi$ and implicitly mean $\pi_{\theta}, \phi_{\varphi}^{\pi_{\theta}}, \psi^{\pi_{\theta}}$, respectively.

\section{Behavioral Actor-Critic}
The policy is trained with the objective of maximizing the expected return and the future behavior function simultaneously. Hence, we can write the objective function as following:\\
\begin{equation}
J(\pi) = \E_{s\sim p, a \sim\pi} [\sum_{t_0}^{\infty} (\gamma^t r(s_t,a_t)  + \alpha \gamma^{t+1} \psi( s_{t+1},a_{t+1}))]
\label{eq-objective} \end{equation}
The key idea behind maximizing the policy future behavior function is that the expected distance between each state-action pair and its representation is maximized if and only if the policy generates novel data distributions. The latter results in a new autoencoder network that learns to effectively represent the new distributions. In other words, the policy is expected to change its behavior consistently to achieve sufficient exploration of the environment. Also, note that the behavior function is regularized by a scalar $\alpha$ to control how important the exploration is; we use larger $\alpha$ when dealing with environments where policy gradient methods often get stuck. In summary, the previous objective maximization leads to policies that follow optimal strategies.

\subsection{Behavioral Policy Iteration}
We adopt the classical actor-critic framework  and prove that executing policy evaluation and policy improvement steps iteratively leads to convergence to the optimal policy. We, then, extend our theoretical analysis to derive the behavioral actor-critic algorithm.

The policy evaluation step is responsible for evaluating the value of a policy according to the objective in Equation  \eqref{eq-objective}. Let $\pi$ be a fixed policy and define the Bellman operator as follows:\\
\begin{equation}
\mathcal T^\pi Q(s,a)\overset{\text{def}}{=} r(s,a)+ \E_{s'\sim p, a'\sim \pi}[\alpha \psi(s',a')+\gamma Q(s',a')]
\label{bellman} \end{equation}
The former definition motivates the definition of the behavioral $Q$ value. We state this in the following proposition.
\begin{proposition}(behavioral Q-value)
Consider the soft Bellman backup operator $\mathcal T^\pi$ in \eqref{bellman} and an intilization of behavioral Q-value $Q_0: \mathcal S \times \mathcal A\rightarrow \mathbb {R}$ s.t. $|\mathcal A|,|\mathcal S|< \infty$. Define the sequence: $(Q_k)_{k\geq 0}: Q_{k+1} = \mathcal T^\pi Q_k$. Then $\lim\limits_{k \to \infty} Q_k=Q^\pi$, where $Q^\pi$ is the behavioral Q-value of $\pi$.
\end{proposition}
\begin{proof} 
See Appendix A.1.
\end{proof}

The policy improvement step produces a new policy that maximizes the old policy's expected $Q$ value for any given state. Hence, it can be defined as: 
\begin{equation} 
\forall s \in \mathcal S:~~\pi_{new}(.,s)=\argmax_\pi \E_{a\sim\pi}[Q^{\pi_{old}}(s,a)].
\label{pinew}\end{equation}
We show that the new policy has a higher value than the old policy when following the augmented objective. More formally, we demonstrate this result in the following proposition.  
\begin{proposition}\label{c}(Policy improvement) Assume that $\pi_{old}$ and $\pi_{new}$ are the current policy, and the new policy defined in  \eqref{pinew}, respectively. Then $Q^{\pi_{new}}(s,a) \geq Q^{\pi_{old}}(s,a)$ for all $(s,a) \in \mathcal S \times \mathcal A$.
\end{proposition}
\begin{proof} 
See Appendix A.2.
\end{proof}

Consequently, executing behavioral policy iterations results in a monotonically increasing sequence $Q^{\pi_k}$ which is bounded since $r(s,a), \psi(s,a)$ are bounded for all $(s,a) \in \mathcal S \times \mathcal A$; this results in the convergence of the sequence for some $\pi^*$. Moreover, according to proposition \autoref{c}, \\ $ Q^{\pi^*}(s,a)> Q^{\pi}(s,a),~ \forall \pi \neq \pi^*, (s,a) \in \mathcal S \times \mathcal A$. Thus, $J(\pi^*)>J(\pi)$, i.e. $\pi^*$ is an optimal policy.  

\subsection {Algorithm}
The theoretical analysis discussed above demonstrated the compatibility of our model with the classical reinforcement learning architectures. Based on this, we used an off-policy actor-critic approach. To that end, we will use function approximators for both the behavioral Q-function and the policy, and optimize both networks by stochastic gradients. Specifically, we adopt the framework presented in \citep{hasselt2010doubleQ,fujimoto2018addressing} by using double critic networks $Q_{\omega_{1,2}}$ where each critic independently approximates the policy's action value, and an actor network $\pi_{\theta}$. Furthermore, separate copies of the critic networks: $ Q'_{\omega'_{1,2}}$ are kept as target networks for stability. These networks are updated periodically using the critics: $ Q_{\omega_{1,2}}$ which are regulated by a weighting parameter $\tau$. Conventionally, the autoencoder is periodically trained on state-action pairs induced by the policy when performing a rollout to minimize $\psi$. 

The behavioral Q-functions parameters can be trained to minimize the following loss
\begin{equation}
\begin{split}
\mathcal L_Q(\omega_{1,2})=&\mathbb{E} _{(s_t,a_t)\sim \mathcal B}[(r(s_{t},a_{t})+\\
& \mathbb{E}_{s_{t+1}\sim \mathcal B, a_{t+1} \sim \pi_{cur}}   [  \alpha \psi^{\pi_{cur}}(s_{t+1},a_{t+1})+\gamma \min_{j=1,2} Q'_{\omega'_j}(s_{t+1},a_{t+1})]-Q_{\omega_{1,2}}(s_t,a_t))^2]
\label{qloss}\end{split}\end{equation}
The replay buffer $\mathcal B$ stores past transitions, and $\pi_{cur}$ refers to the current policy of interest. Consequently, $\mathcal L$ can be optimized by stochastic gradients:
\begin {equation}\begin{split}
\hat{\nabla}_{\omega_{1,2}} \mathcal L_Q(\omega_{1,2})=&-2 \nabla_{\omega_{1,2}} Q_{\omega_{1,2}}(s_{t},a_{t})\\
&(r(s_{t},a_{t})+  \alpha \psi^{\pi_{cur}}(s_{t+1},a_{t+1})+\gamma \min_{j=1,2} Q'_{\omega'_j}(s_{t+1},a_{t+1})-Q_{\omega_{1,2}}(s_t,a_t))
\label{grad-q}\end{split}\end{equation}
Finally, the policy parameters can be learned by directly maximizing the expected behavioral Q-value in Equation \eqref{j-grad}:\\
\begin{equation}
J(\pi)=\mathbb E_{s\sim \mathcal B, a\sim \pi}[\min_{j=1,2}Q_{\omega_j}(s,a)]
\label{j-grad}\end{equation}
That is, we follow the gradient of $J$:
\begin{equation}
{\nabla}_{\theta}J=\mathbb E_{s\sim \mathcal B, a\sim \pi}[\min_{j=1,2}Q_{\omega_j}(s,a)\nabla_{\theta}\log \pi(a|s)]
\label{j-grad2}\end{equation}
Alternatively, we can apply the reparameterization trick mentioned in \citep{kingma2013auto}, by assuming that $a_t=g_{\theta}(\epsilon_t; s_t)$, where $\epsilon_t$ is an input vector sampled from Gaussian distribution. This allows us to optimize $J$ by following the approximate gradient:
\begin{equation}
\hat{\nabla}_{\theta} J= \nabla_{\theta}g_{\theta}(\epsilon_t; s_t) \nabla_a \min_{j=1,2}Q_{\omega_j}(s_t,a)|_{a=g_{\theta}(\epsilon_t; s_t)}
\label{j-grad3}\end{equation}
\begin{algorithm}[!ht]
 \begin{algorithmic}
 \STATE Initialize the critics, actor, and autoencoder networks, $Q_{\omega_{1,2}}$, $\pi_{\theta}$ and $\phi_{\varphi}$, respectively, and a replay buffer $\mathcal{B}$ that stores past transition samples for training.
 \STATE {\bf for} each iteration $i$ {\bf do}:
 \STATE \ \ \ \ {\bf if} $i$ mod $autoencoder ~update frequency$ {\bf do}:
 \STATE \ \ \ \ \ \ \ \ Obtain trajectory $\mathcal{J}$.
 \STATE \ \ \ \ \ \ \ \ Update autoencoder parameters:
 \STATE \ \ \ \ \ \ \ \ $\varphi\leftarrow \varphi-\lambda_{\varphi}\nabla_{\varphi}\psi^{\mathcal{J}}$
 \STATE \ \ \ \ {\bf for} each environment step $t$ {\bf do}:
 \STATE \ \ \ \ \ \ \ \ Sample and perform $a_t\sim \pi_{\theta}(s_t,\cdot)$
 \STATE \ \ \ \ \ \ \ \ Add $(s_t,a_t,s_{t+1},r_t)$ to $\mathcal{B}$
  \STATE \ \ \ \ \ \ \ \ {\bf for} each update step {\bf do}:
 \STATE \ \ \ \ \ \ \ \ \ \ \ \ Sample a random batch $\mathcal{D}$ from $\mathcal{B}$.
 \STATE \ \ \ \ \ \ \ \ \ \ \ \ Update critics parameters:

 \STATE \ \ \ \ \ \ \ \ \ \ \ \ \ \ \ \ $\omega_{j}\leftarrow\omega_j-\lambda_{\omega_j}\nabla_{\omega_j}\mathcal L (\omega_j)$ for $j\in\{1,2\}$
 \STATE \ \ \ \ \ \ \ \ \ \ \ \ Update actor parameters:
 \STATE \ \ \ \ \ \ \ \ \ \ \ \ \ \ \ \ $\theta\leftarrow\theta+\lambda_{\theta}\nabla_{\theta}J$
 \STATE \ \ \ \ \ \ \ \ \ \ \ \ Update target networks parameters:
 \STATE \ \ \ \ \ \ \ \ \ \ \ \ \ \ \ \ $\omega'_{j}\leftarrow\tau\omega_j+ (1-\tau)\omega'_j$ for $j\in\{1,2\}$
 \end{algorithmic}
\caption{Behavior-Guided Actor-Critic Algorithm.}
\label{algorithm-1}
\end{algorithm}
\section{Experiments}
\subsection{Empirical Evaluation}

We measure the performance of our algorithm\footnote{Code available at: \url{https://github.com/AmmarFayad/Behavioral-Actor-Critic}} on a suite of PyBullet \citep{tan2018sim} continuous control tasks, interfaced through OpenAI Gym\citep{brockman2016openai}. While many previous works utilized the Mujoco \citep{todorov2012mujoco} physics engine to simulate the system dynamics of these tasks, we found it better to evaluate BAC on benchmark problems powered by PyBullet simulator since it is widely reported that PyBullet problems are harder to solve \citep{tan2018sim} when compared to Mujoco. Also, Pybullet is license-free, unlike Mujoco that is only available to its license holders.
\begin{figure}[h!]
\centering
  \begin{subfigure}[b]{0.3\linewidth}
    \includegraphics[width=\linewidth]{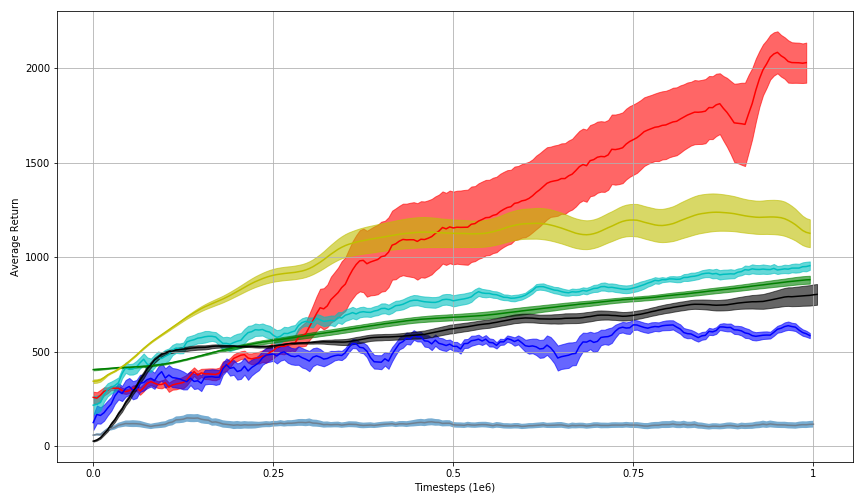}
     \caption{Ant}
  \end{subfigure}
\begin{subfigure}[b]{0.3\linewidth}
    \includegraphics[width=\linewidth]{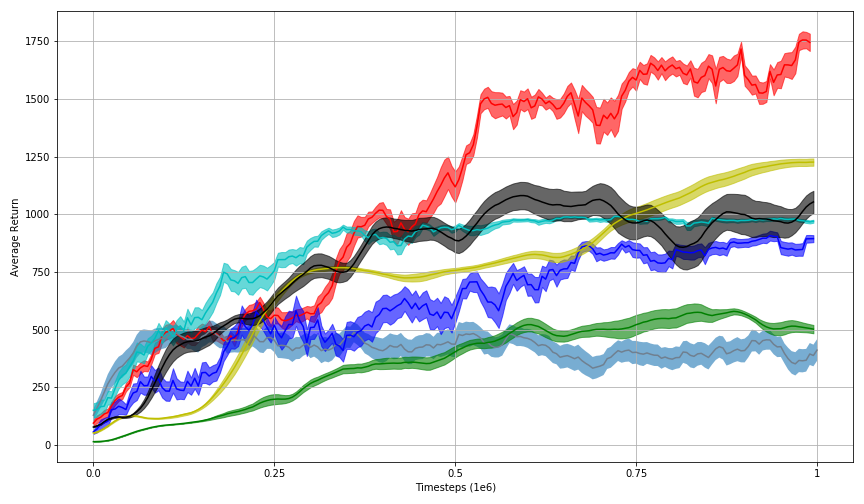}
     \caption{Walker2D}
\end{subfigure}
\begin{subfigure}[b]{0.3\linewidth}
    \includegraphics[width=\linewidth]{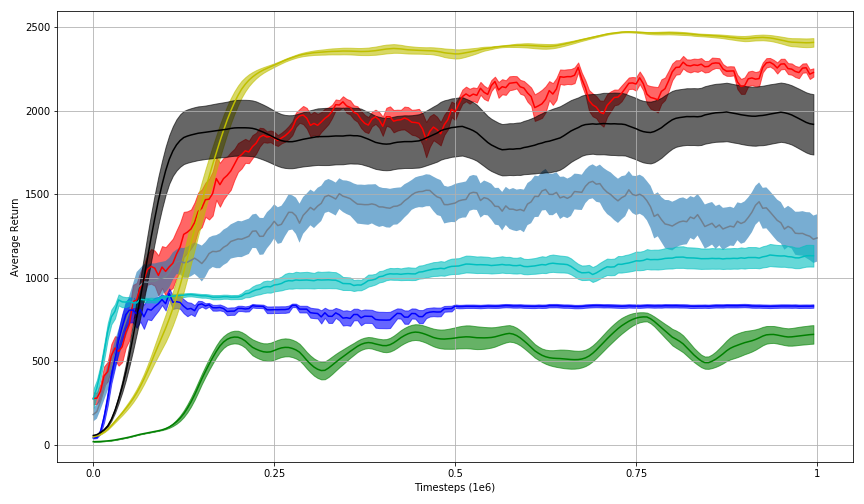}
     \caption{Hopper}
\end{subfigure}
\begin{subfigure}[b]{0.3\linewidth}
    \includegraphics[width=\linewidth]{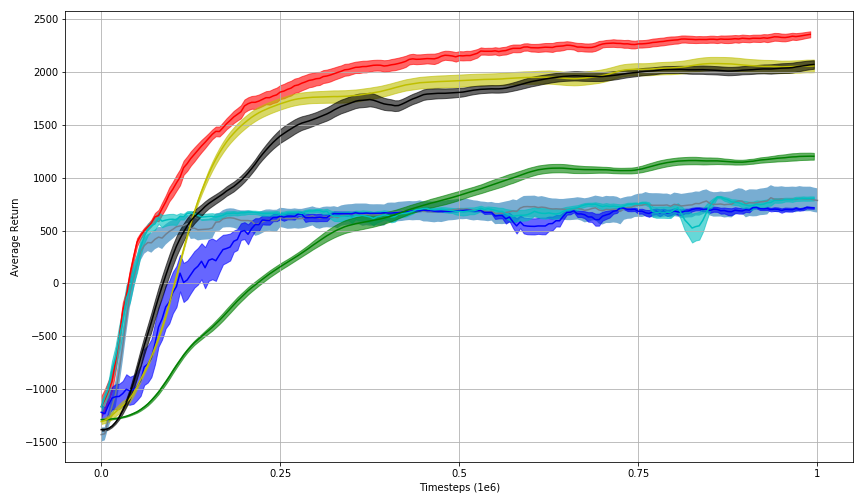}
     \caption{HalfCheetah}
  \end{subfigure} 
\begin{subfigure}[b]{0.3\linewidth}
    \includegraphics[width=\linewidth]{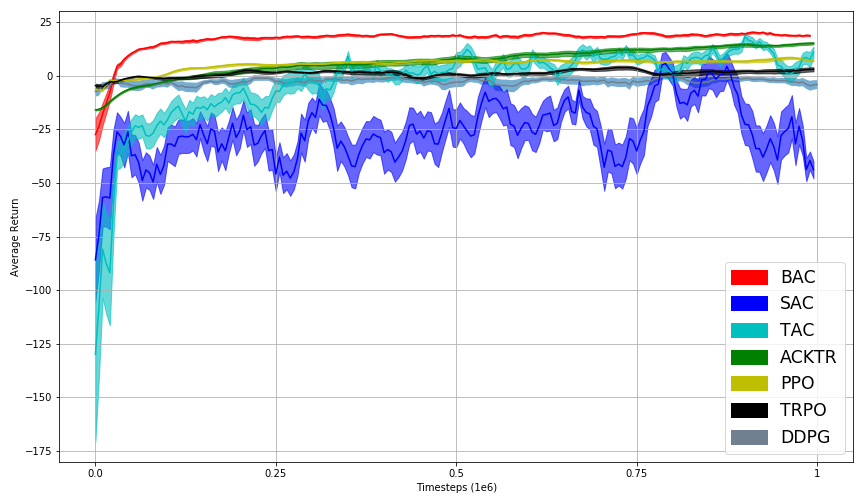}
     \caption{Reacher}
\end{subfigure}
\caption{Learning curves for the OpenAI gym continuous control tasks. The shaded region represents quarter a standard deviation of the average evaluation. Curves are smoothed uniformly for visual clarity.}
\label{1}
\end{figure}
We compare our method to soft actor critic (SAC)\citep{Haarnoja2018SoftAO}; Tsallis actor-critic (TAC)\citep{chen2019off}, a recent off-policy algorithm for learning maximum entropy policies, where we use the implementation of authors\footnote{ \url{https://github.com/haarnoja/sac}} \footnote{ \url{ https://github.com/yimingpeng/sac-master}};  proximal policy optimization (PPO)\citep{schulman2017proximal}, a stable and efficient on-policy policy gradient algorithm; deep deterministic policy gradient (DDPG)\citep{Lillicrap2015ContinuousCW}; trust region policy optimization (TRPO)\citep{Schulman2015TrustRP}; and Actor Critic using Kronecker-Factored Trust Region (ACKTR)\citep{Wu2017ScalableTM}, as implemented by OpenAI’s baselines repository \footnote{ \url{https://github.com/openai/baselines}}. Each task is run for 1 million time steps and the average return of 15 episodes is reported every 5000 time steps. To enable reproducibility,  each experiment is conducted on 3 random seeds of Gym simulator and network initialization.

In order for all $\psi$ features to contribute proportionally to the behavior function, it seems convenient to enforce feature scaling since state-action pairs may correspond to a broad range of $\psi$ values. One possible solution is to normalize $\psi$ by considering $\bar{\psi}(s,a)=\frac{\psi(s,a)}{\max\psi(s,a)}$. 
As a matter of fact, some environments are highly stochastic in terms of transition dynamics which makes learning unstable, not to mention the direct effect of the behavioral Q-function on the scale of gradient that may lead to drastic updates of the policy and even bad convergence. Motivated by this idea, we limit the policy gradient to an appropriate range by a global norm $\mathcal N$. Specifically, we applied this technique to HalfCheetah Bullet environment setting $\mathcal N=3$.

\begin{table*}[t]
\centering
\caption{Max Average Return over trials of 1 million time steps. the best performing algorithm for each task is bolded. $\pm$ corresponds to a single standard deviation over trials. }
\label{1results}
\begin{center}
\begin{scriptsize}
\begin{tabular}{lcccccc}
\toprule
\bf{Algorithm} & \bf{HalfCheetah} & \bf{Reacher} & \bf{Ant} & \bf{Walker2D} & \bf{Hopper} \\
\midrule
BAC 		& \bf{2369.15 $\pm$ 104.53}	 & \bf{23.42 $\pm$ 0.49} 	& \bf{2122.81 $\pm$ 392.34}	& \bf{1814.93 $\pm$ 153.59} 	& 2384.24 $\pm$ 111.42   \\  
DDPG 	& {830.13 $\pm$ 467.08}		 & 1.68 $\pm$ 5.18 				& 177.33 $\pm$ 112.56		& 586.84 $\pm$ 154.66		& 1650.68 $\pm$ 392.73  \\
SAC 		& {777.37 $\pm$ 60.5} 		 & 24.24 $\pm$ 5.56				& 685.53 $\pm$ 26.49 		& 932.92 $\pm$ 74.06		& 972.28 $\pm$ 41.23  \\
TAC 		& {844.58 $\pm$ 45.08} 		 & 25.6 $\pm$ 6.97 			& 974.0 $\pm$ 127.21 		& 1000.54 $\pm$ 1.3		& 1160.17 $\pm$ 289.93  \\
ACKTR 	& {1210.0 $\pm$ 126.75}		 & 14.63 $\pm$ 1.36 	   		& 830.54 $\pm$ 58.71		&642.33 $\pm$ 143.2			& 754.33 $\pm$ 96.84 \\
PPO		& {2081.41$\pm$ 251.66}		 & 9.05 $\pm$ 1.95 	   			& 1240.08 $\pm$ 390.54		& 1231.21 $\pm$ 62.37	& \bf{2477.98 $\pm$ 16.71} \\
TRPO 	& {2082.81$\pm$ 151.81}			 &4.12 $\pm$ 3.7	   			& 806.12 $\pm$ 222.99			& 1096.48 $\pm$ 224.99	& 1999.87 $\pm$ 701.62 \\

\bottomrule
\end{tabular}
\end{scriptsize}
\end{center}
\vskip -0.2in
\end{table*}

Our results are shown in Table (\ref{1results}) and learning curves in Figure (\ref{1}). Based on that, we can deduce that BAC generally outperforms all other algorithms in terms of both final performance and learning speed.

\begin{wrapfigure}{R}{0.45\textwidth}
\centering
\includegraphics[width=1\linewidth]{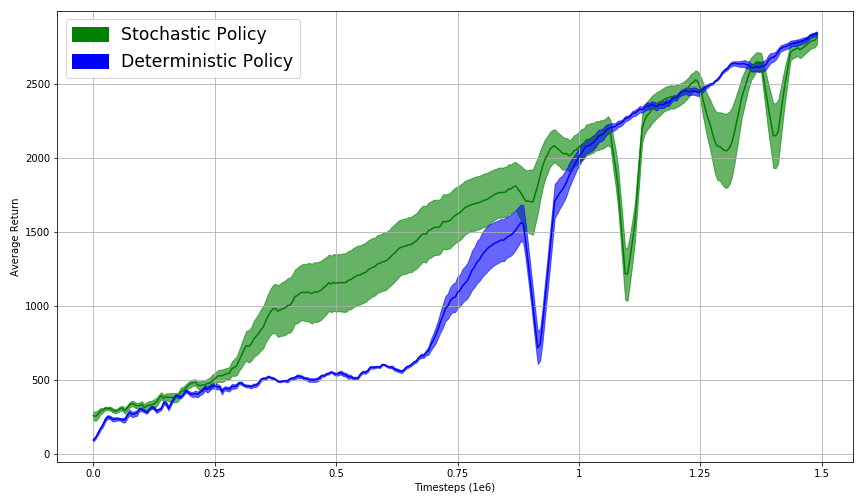}
	\caption[width=\columnwidth]{
		Stochastic vs. Deterministic BAC.
	}\label{2}
\end{wrapfigure} 
\subsection{Stochastic Policy vs Deterministic Policy}
In our experiments, behavior-guided actor critic was encouraged to learn a stochastic policy. However, given that our method is not restricted to a specific policy type, it seems tempting to compare the standard BAC to its deterministic variant, where we seek to optimize the objective in Equation \eqref{eq-objective} by following a deterministic policy. For this sake, we conducted our comparative evaluation on the high dimensional Ant Bullet environment and the experiments were run for 1.5 million steps.

One point worthy of noting is that autoencoders may fail to correctly represent novel and promising state-action pairs if they are in the neighborhood of a frequently visited state-action pair; this may not allow the policy to try all promising paths and lead to slow learning. To fill that gap, we considered adding small noise to the actions selected by the deterministic policy.

Figure (\ref{2}) shows the superiority of the deterministic variant in terms of stability, but no difference in the final performance when compared to the stochastic BAC meaning that the behavior function is achieving sufficient exploration regardless of the policy type and leading to a state-of-the-art performance along with the proposed architecture of BAC.

\subsection{Behavior Temperature}
Table (\ref{alpha}) shows behavior temperatures $\alpha$ used in our experiments:
\begin{table*}[h]
\centering
\caption{Temperature parameter}
\label{alpha}
\begin{center}
\begin{footnotesize}
\begin{tabular}{lcccccc}
\toprule
\bf{Environment} & \bf{Hopper} & \bf{Walker2D} & \bf{Reacher}    & \bf{Ant} & \bf{HalfCheetah}\\
\midrule
\bf{Static $\alpha$} 	      & 0.05	    & 0.04 	  & 0.05     	& 0.2 & 0.03 \\

\bottomrule
\end{tabular}
\end{footnotesize}
\end{center}
\vskip -0.2in
\hskip -0.4in
\end{table*}

We further explore a variant of BAC that dynamically weights the priority given to the discounted long-term reward function against the behavior function. We make use of a more expressive and tractable approach to endow our method with more efficient exploration and performance while overcoming the brittleness in setting $\alpha$. Specifically, the algorithm follows the performance gradient when it is making progress, and seeks to maximize the behavior function if stuck in a local optimum. To address this issue, consider the weighted objective $J(\pi)$ as following:
\begin{wrapfigure}{R}{0.45\textwidth}
\centering
\includegraphics[width=1\linewidth]{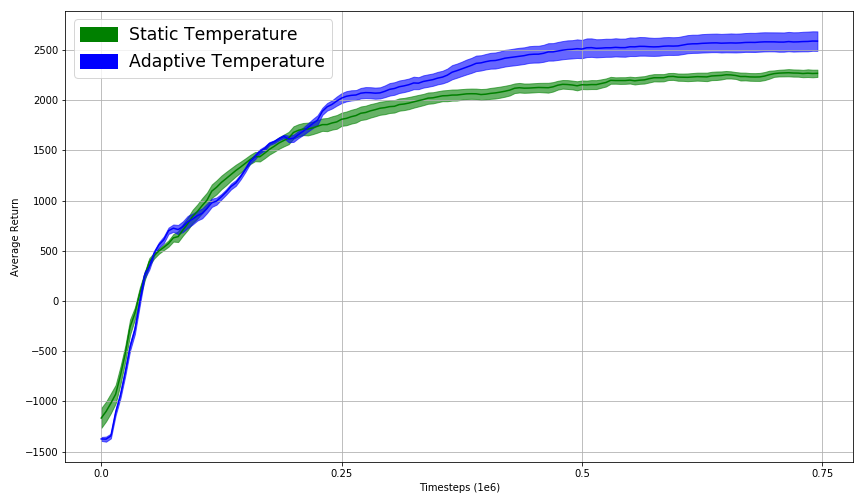}
	\caption[width=\columnwidth]{
		Adaptive vs. Static Temperature
	}\label{alpha1}
\end{wrapfigure}
\begin{equation}
J(\pi) = \E_{s\sim p, a \sim\pi} [\sum_{t_0}^{\infty}\gamma^t ( \alpha r(s_t,a_t)  + \gamma(1- \alpha) \psi( s_{t+1},a_{t+1}))]
\end{equation}
We make use of a neural network $\mathcal Q_{\eta}$ ($\mathcal Q$ for simplicity) to estimate the quality of an action in a certain state with respect to the policy of interest and future rewards only. Similar to the critic, $\mathcal Q$ seeks to minimize the loss $\mathcal L=\mathbb E[||\mathcal Q' -\mathcal Q||_2^2]$, where $\mathcal Q=\mathcal Q(s,g_{\theta}(\epsilon,s))$ and
$ \mathcal Q'= r + \gamma \mathcal Q(s',g_{\theta}(\epsilon,s'))$ s.t. $s,s'\sim \mathcal D$ .

The intuition behind using $\mathcal Q$ to come up with a formula of $\alpha$, is that when the distance between $\mathcal Q$ and its target is small, the algorithm is not making progress and is probably stuck in a local optimum, and vice versa. We draw an inspiration from this idea to define $\alpha$ as following:  
\begin{equation}
\alpha=\sigma\left(\omega\frac{ \overline{|\mathcal Q- \mathcal Q'|}-\min|\mathcal Q-\mathcal Q'|}{\max|\mathcal Q-\mathcal Q'|-\min|\mathcal Q-\mathcal Q'|}\right)
\label{alpha2}
\end{equation}
where $\sigma$ is the sigmoid function and $\omega=10$ for our convenience.
We empirically evaluated this method and compared it to the standard BAC where the experiments were run for 0.75 million steps on HalfCheetah Bullet environment and results were shown in Figure (\ref{alpha1}). 
\section{Conclusion}
Behavior-guided actor critic is designed to promote exploration and avoid local optima that most recent algorithms struggle with. Specifically, BAC benefits from the discrepancy between the agent-relative representation and the actual features of the state-action pairs to formulate a function that describes the agent's behavior and consequently maximizes it to avoid getting stuck in a certain region of the solution space. To achieve this, we draw inspiration from various RL frameworks and adopt an off-policy actor critic method to optimize the conventional RL objective along with the behavior function. We also provide theoretical analysis and prove that our method converges to the optimal policy allowing us to formalize the BAC algorithm. BAC has shown high competence and proved to be a competitive candidate to solve challenging, high-dimensional RL tasks.  Our method can also be extended to various learning frameworks. We believe that investigating the relative importance $\alpha$ of the behavior function and finding an adaptive schedule would be fruitful to consider in future works.
\section{Broader Impact}
This work does not currently present any foreseeable societal consequence.

\bibliographystyle{apa}
\end{document}